\documentclass[sigconf,nonacm]{aamas}

\usepackage{balance} 
\usepackage{amsmath}
\usepackage{amsthm}
\usepackage{amssymb}
\usepackage[capitalize]{cleveref}
\usepackage{algorithm}
\usepackage{algorithmic}
\usepackage{subfig}
\usepackage{siunitx}
\usepackage{afterpage}

\usepackage{vector}

\graphicspath{{../../images/}}

\newif\ifarxiv
\arxivtrue

\title[HIPG-MARL]{Health-Informed Policy Gradients for Multi-Agent Reinforcement Learning}

\author{Ross E. Allen}
\affiliation{
  \institution{MIT Lincoln Laboratory}
  \city{Lexington}
  \state{Massachusetts}}
\email{ross.allen@ll.mit.edu}

\author{Jayesh K. Gupta}
\affiliation{
  \department{Computer Science}
  \institution{Stanford University}}
\email{jkg@cs.stanford.edu}

\author{Jaime Pena}
\affiliation{
  \institution{MIT Lincoln Laboratory}
  \city{Lexington}
  \state{Massachusetts}}
\email{jdpena@ll.mit.edu}

\author{Yutai Zhou}
\affiliation{
  \institution{MIT Lincoln Laboratory}
  \city{Lexington}
  \state{Massachusetts}}
\email{yutai.zhou@ll.mit.edu}

\author{Javona White Bear}
\affiliation{
  \institution{MIT Lincoln Laboratory}
  \city{Lexington}
  \state{Massachusetts}}
\email{jwbear@ll.mit.edu}

\author{Mykel J. Kochenderfer}
\affiliation{
  \department{Aeronautics \& Astronautics}
  \institution{Stanford University}}
\email{mykel@stanford.edu}

\begin{abstract}
This paper proposes a definition of system health in the context of multiple agents optimizing a joint reward function. 
We use this definition as a credit assignment term in a policy gradient algorithm to distinguish the contributions of individual agents to the global reward. %
The health-informed credit assignment is then extended to a multi-agent variant of the proximal policy optimization algorithm and demonstrated on particle and multiwalker robot environments that have characteristics such as system health, risk-taking, semi-expendable agents, continuous action spaces, and partial observability. %
We show significant improvement in learning performance compared to policy gradient methods that do not perform multi-agent credit assignment\ifarxiv{\footnote{DISTRIBUTION STATEMENT A. Approved for public release. Distribution is unlimited.

This material is based upon work supported by the Under Secretary of Defense for Research and Engineering under Air Force Contract No. FA8702-15-D-0001. Any opinions, findings, conclusions or recommendations expressed in this material are those of the author(s) and do not necessarily reflect the views of the Under Secretary of Defense for Research and Engineering.

\copyright 2020 Massachusetts Institute of Technology.

Delivered to the U.S. Government with Unlimited Rights, as defined in DFARS Part 252.227-7013 or 7014 (Feb 2014). Notwithstanding any copyright notice, U.S. Government rights in this work are defined by DFARS 252.227-7013 or DFARS 252.227-7014 as detailed above. Use of this work other than as specifically authorized by the U.S. Government may violate any copyrights that exist in this work.}}\fi.%
\end{abstract}


\newcommand{\BibTeX}{\rm B\kern-.05em{\sc i\kern-.025em b}\kern-.08em\TeX}

\newtheorem{lemma}{Lemma}

\newtheoremstyle{propertystyle}
  {\topsep}
  {\topsep}
  {\itshape}
  {0pt}
  {\bfseries}
  {}
  { }
  {\thmname{#1}\thmnumber{ #2}.\thmnote{ (#3)}}
\theoremstyle{propertystyle} \newtheorem{property}{Property}

\begin{document}


\pagestyle{fancy}
\fancyhead{}


\maketitle 


\section{Introduction}\label{sec:introduction}

Autonomous robotic systems are commonly employed for tasks described as dull, dirty, and dangerous. %
Multi-robot systems are particularly well suited for dirty and dangerous tasks as they are robust to single-agent degradation and failures. %
Such degradation can arise from damage to an agent's sensors or actuators, thus limiting the agent's ability to observe the environment and constricting the actions it may take. %
We will use the term \emph{system health} to refer to an agent's current state of degradation relative to its nominal capabilities.

While multi-agent systems may be an attractive solution for a range of real-world tasks, developing distributed decision-making and control policies is still challenging. %
Decision-making in multi-agent systems can be modeled as a decentralized partially observable Markov decision process (Dec-POMDP)~\cite{kochenderfer2015decision}. %
In general, computing an optimal policy is \emph{NEXP-complete}~\cite{bernstein2002complexity}. %
While various solution techniques based on heuristic search~\cite{szer2005optimal,spaan2011scaling,oliehoek2013approximate}
and dynamic programming~\cite{hansen2004dynamic,boularias2008exact,Balaban2019} exist, they tend to quickly become intractable as the number of agents, states, and actions increase.

One approach to approximate solutions to such problems is to use deep reinforcement learning (deep-RL) \cite{vinyals2019grandmaster,openai2019five,gupta2017cooperative,lowe2017multi,foerster2018counterfactual}. %
However, a fundamental challenge is the \emph{multi-agent credit assignment problem}. %
When there are multiple agents acting simultaneously toward a shared objective, it is often unclear how to determine which actions of which agents are responsible for the joint reward. There may be strong inter-dependencies between the actions of different agents and long delays between joint actions and their eventual rewards~\cite{wolpert2002optimal}.

This work fuses the concept of system health with deep reinforcement learning and provides three contributions to the field of multi-agent decision-making. %
First, we outline a definition and properties for the concept of system health in the context of Markov decision processes. %
These definitions provide the framework for a subset of Dec-POMDPs that can be used to analyze multi-agent systems operating in hazardous and adversarial environments. %
Second, we use the definition of health to formulate a multi-agent credit assignment algorithm to be used to accelerate and improve multi-agent reinforcement learning. 
Third, we apply the health-informed crediting algorithm within a multi-agent variant of proximal policy optimization (PPO)~\cite{schulman2017proximal} and demonstrate significant learning improvements compared to existing algorithms in simulated environments involving legged and particle robots with continuous action spaces.

\section{Related Work}\label{sec:related_work}

Applying deep-RL to multi-agent decision-making is an active area of research. %
Hernandez-Leal, Kartal, and Taylor~\cite{hernandez2019survey} provide a comprehensive survey of the field. %
Gupta et al.~\cite{gupta2017cooperative} demonstrated how algorithms such as TRPO, DQN, DDPG, and A3C can be extended to a range of cooperative multi-agent problems. %
Lowe et al.~\cite{lowe2017multi}  
developed multi-agent deep deterministic policy gradients (MADDPG) that was capable of training in cooperative and competitive environments. 

Multi-agent reinforcement learning is challenging due to the problems of \emph{non-stationarity} and \emph{multi-agent credit assignment}. %
The non-stationary problem arises when a learning agent assumes all other learning agents are part of the environment dynamics. Since the individual agents are continuously changing their policies, the environment dynamics from the perspective of any one agent are continuously changing, thus breaking the Markov property~\cite{hernandez2017survey}.
While Lowe et al.~\cite{lowe2017multi} attempt to address the non-stationary problem, MADDPG is still shown to become ineffective at learning for systems with more than three or four agents. %

Gupta et al.~\cite{gupta2017cooperative} partially address the non-stationary problem through \emph{parameter sharing} whereby groups of homogeneous agents use identical copies of parameters for their local policies. 
Terry et al.~\cite{terry2020parameter} provide a theoretical analysis of how information centralization offered by parameter sharing alleviates some of the non-stationarity problem. 
However, parameter sharing techniques don't resolve the second fundamental challenge of multi-agent learning: multi-agent credit assignment---i.e.\ the challenge of identifying which actions from which agent at which time were most responsible for the overall performance (i.e. returns) of the system. %
Gupta et al. avoid explicit treatment of this problem by focusing on environments where the joint rewards can be decomposed into local rewards. However, in general, such local reward structures are not guaranteed to optimize joint returns~\cite{wolpert2002optimal}.

Wolpert and Tumer~\cite{wolpert2002optimal} 
developed the Wonderful Life Utility (WLU) and Aristocrat Utility (AU), which are forms of ``difference rewards''. %
Both WLU and AU attempt to assign credit to individual agents' actions by taking the difference of utility received versus utility that would have been received had a different action been taken by the agent. %
The comparison between actual returns and hypothetical returns is sometimes referred to as  \emph{counterfactual} learning~\cite{foerster2018counterfactual}. %
Predating most of the advancements in deep reinforcement learning, Wolpert and Tumer's work was restricted to small decision problems that could be handled in a tabular fashion~\cite{wolpert2002optimal,tumer2002learning}. %

Foerster et al.~\cite{foerster2018counterfactual} formulated an aristocrat-like crediting method that was able to leverage a deep neural network state-action value function (Q-value) within a policy gradient algorithm, referred to as counterfactual multi-agent (COMA) policy gradients. %
Using deep neural networks, they enabled crediting in large or continuous state spaces. %
Subsequent work on the multi-agent credit assignment led to the QMIX and Maven algorithms which factorized the Q-value across agents and demonstrated improved performance over COMA \cite{rashid2018qmix,mahajan2019maven}. However COMA, QMIX, and Maven required enumeration over all actions and were thus restricted to problems with discrete action spaces. 

Others have posed multi-agent, health-aware decision problems similar to the one we give in Section~\ref{sec:problem_statement} \cite{omidshafiei2016health,oliehoek2008cross}. 
These prior works use a planning-based approach which assumes detailed, a priori knowledge of the underlying Dec-POMDP's dynamics. This is fundamentally different from the learning-based approach we present in Section~\ref{sec:approach}.

\section{Problem Statement}\label{sec:problem_statement}

The problems presented in this work can be modeled as decentralized partially observable Markov decision processes (Dec-POMDPs), which are defined by the tuple $(\mathcal{I}, \mathcal{S}, \mathcal{A}_i, \mathcal{Z}_i, T, R)$. %
The set $\mathcal{I}$ represents a finite set of $n$ agents. %
The set $\mathcal{S}$ is the joint state space of all agents (finite or infinite). Assuming states are described in a vector form, let $\vect{s} \in \mathcal{S} \subseteq \mathbb{R}^m$ be a specific state of the system. 
The set $\mathcal{A}_i(\vect{s})$ is the action space of the $i$th agent in joint state $\vect{s}$. The vector $\vect{u}_t = \left(a_{1,t}, a_{2,t}, \ldots, a_{n,t} \right)$ represents a joint action at time $t$ where $a_{i,t} \in \, \mathcal{A}_i(\vect{s}_t)$. %
The set $\mathcal{Z}_i(\vect{s})$ is the set of observations for the $i$th agent in joint state $\vect{s}$. The vector $\vect{o}_t = \left(z_{1,t}, z_{2,t}, \ldots, z_{n,t} \right)$ represents a joint observation at time $t$ where $z_{i,t} \in \, \mathcal{Z}_i(\vect{s}_t)$. %
The transition function $T \left( \vect{s}' \mid \vect{s},\vect{u} \right)$ is the probability density associated with arriving in state $\vect{s}'$ given the joint action $\vect{u}$ was taken in state $\vect{s}$. %
The reward function $R(\vect{s}, \vect{u})$ gives the immediate reward for taking the joint action $\vect{u}$ while in state $\vect{s}$. %
The vector $\tau_{i,t} = \left(z_{i,1}, a_{i,1}, z_{i,2}, a_{i,2}, \ldots, z_{i,t} \right)$ represents the observation-action history for agent $i$ up to time $t$. %
Using notation similar to Foerster et al.~\cite{foerster2018counterfactual}, we represent group-wide joint variables in bold and joint quantities that exclude a particular agent with the term $\neg i$.

In order to solve the Dec-POMDP we seek a \emph{joint policy} $\vect{\pi}_{\vect{\theta}} (\vect{u} \mid \vect{s})$, composed of a set of \emph{local policies} $\pi_{\theta_i}(a_i \mid \tau_i)$, that 
maximizes the discounted joint returns $G_t = \sum_{l=0}^{t_f-t} \gamma^{l} r_{t+l}$, where $\gamma$ is the discount factor, $r$ is the empirical joint reward, and $t_f$ is the final time step in an episode or receding horizon. %
The local policies, parameterized by $\theta_i$, map an agent's observation-action history to its next action at each time step. %
For a group of $n$ agents that follow independent stochastic policies\footnote{Independent policies implies that agent $i$'s action at time $t$ is not conditioned on the actions of any other agent at time $t$. Note that, in the case of \emph{parameter sharing} where all agents have identical copies of parameters, $\theta_i = \theta$, the local policies can still be independent.} %
\cite{foerster2018counterfactual}, we have
\begin{equation}\label{eqn:joint_policy_prop}
\vect{\pi} \left( \vect{u} \mid \vect{s}, \vect{\theta} \right) = \prod_{i=1}^{n} \pi_i \left( a_i \mid \tau_i, \theta_i \right).
\end{equation}
Using the common definition of the state value function of policy $\vect{\pi}$, we define $V_{\vect{\pi}}(\vect{s}_t) = \mathbb{E}_{\vect{s}_{t+1:t_f}, \vect{u}_{t:t_f}} \left[ G_t \right]$ \cite{schulman2015high}. %
Similarly, we define the joint state-action value function as $Q_{\vect{\pi}}(\vect{s}_t, \vect{u}_t) = \mathbb{E}_{\vect{s}_{t+1:t_f}, \vect{u}_{t+1:t_f}} \left[  G_t \right]$.

\subsection{System Health}\label{subsec:system_health}

Here we introduce a concept we refer to simply as \emph{system health}, though there are alternative definitions used in the field of prognostic decision making (PDM) \cite{saxena2008metrics,balaban2013development}. %
If we represent the current state of the system with vector $\vect{s}$, then the system health, $\vect{h} \in \mathbb{R}^n$, constitutes a subvector of $\vect{s}$. %
Without loss of generality, we can define the state vector as $\vect{s} = \left(\vect{h}, \vect{p} \right)$, where $\vect{p}$ is the non-health components of the state.
Each element of the the health vector corresponds to the health of an individual agent and is in the interval $\left[0, 1\right]$, where 1 represents full health and 0 represents a fully degraded agent. %
We define the following properties associated with reduction of system health:

\begin{property}{Non-recoverable minimum health.}\label{prop:min_health}
\normalfont
Let $\vect{s}_{h_i = c}$ represent any state vector where the health of agent $i$ is given as $c \in \left[0, 1\right]$. We define the non-increasing nature of the health once minimum health has been reached (i.e. agent death) as follows
	\begin{equation}
	T \left( \vect{s}'_{h_i =\beta} \mid \vect{s}_{h_i = 0}, \vect{u}\right) = 0 \; \text{for} \; \beta > 0
	\end{equation}
\end{property}

\begin{property}{Constriction of reachable set in state space.}\label{prop:state_constriction}
\normalfont
Define the reachable set of joint state $\vect{s}$ and constriction of reachable set as follows
	\begin{align}
	\begin{split}
	\text{Let: } \mathcal{R}(\vect{s}) = \{ \vect{s}' \in \mathcal{S} \mid \exists \, \vect{u}: T \left( \vect{s}' \mid \vect{s}, \vect{u} \right) > 0 \} \\	
	\therefore \mathcal{R}(\vect{s}_{h_i =\beta}) \subseteq \mathcal{R}(\vect{s}_{h_i =\alpha}) \; \text{for} \; \alpha > \beta
	\end{split}	
	\end{align}
\end{property}

\begin{property}{Constriction of available actions in action space.}\label{prop:action_constriction}
\normalfont
Let $\mathcal{A}_i(\vect{s}) \subseteq \mathcal{A}_i$ represent the available action set for agent $i$ when the system is in state $\vect{s}$. Therefore the constriction of the available action set can be described as
	\begin{equation}
	\mathcal{A}_i(\vect{s}_{h_i = \beta}) \subseteq \, \mathcal{A}_i(\vect{s}_{h_i = \alpha}) \; \text{for} \; \alpha > \beta
	\end{equation}
\end{property}

\begin{property}{Constriction of observable set in observation space.}
\label{prop:observe_constriction}
\normalfont
Let $\mathcal{Z}_i(\vect{s}) \subseteq \mathcal{Z}_i$  represent the set of possible observations for agent $i$ when the system is in state $\vect{s}$. Therefore the constriction of observable set can be described as:
	\begin{equation}
	\mathcal{Z}_i(\vect{s}_{h_i = \beta}) \subseteq \, \mathcal{Z}_i(\vect{s}_{h_i = \alpha}) \; \text{for} \; \alpha > \beta
	\end{equation}
\end{property}

To provide real-world intuition about system health, we can consider multi-agent systems composed of physical robots. %
In such a case, system health can be thought of the state of damage or degradation of a physical robot. A health of zero implies a robot has completely `crashed' or been otherwise terminated and Property~\ref{prop:min_health} asserts that the robot will not become operational again. %
Properties~\ref{prop:state_constriction} and~\ref{prop:action_constriction} describe the effect of damaging a robot's actuators, thus partially or completely `crippling' its motion. %
Property~\ref{prop:observe_constriction} describes the effect of damaging a robot's sensors, thus limiting the observations it can make of the world. %

In this paper, we pay particular attention to the case of \emph{binary health states}, i.e. $h_i \in \lbrace 0,1 \rbrace$, where zero health represents a complete constriction of an agent's action space. %
In this case, agents are either fully functional or completely non-operational, where non-operational agents are incapable of selecting actions to interact with the environment. %
This is a natural formulation for many multi-agent problems, such as multi-agent computer games, where agents maintain their full functionality up until the moment they are non-operational.

\section{Approach}\label{sec:approach}

This section proposes a counterfactual learning algorithm that seeks to address the multi-agent credit assignment problem for systems that embody the characteristics of system health. %
We choose to restrict our scope to policy gradient methods because of their scalability to large and continuous state and action spaces~\cite{duan2016benchmarking}. %
We develop a health-informed policy gradient in \cref{subsec:health_marl_pg} and use it to propose a new multi-agent PPO variant in \cref{subsec:health_marl_ppo}. 


In general, each agent may be learning its own individual policy $\pi_{\theta_i}$, however this can lead to a \emph{non-stationary} environment from the perspective of any one agent and confound the learning process~\cite{hernandez2017survey}. %
Instead, for this paper, we assume that agents are executing identical copies of the same policy $\pi_{\theta}$ in a decentralized fashion, referred to as parameter sharing~\cite{gupta2017cooperative,huttenrauch2018local,terry2020parameter}. 
We adopt a \emph{centralized learning, decentralized execution} architecture whereby training data can be centralized between training episodes even if no such centralization of information is possible during execution---a common approach in the multi-agent RL literature~\cite{gupta2017cooperative,lowe2017multi,foerster2018counterfactual}. 


\subsection{Health-Informed Multi-Agent Policy Gradients}\label{subsec:health_marl_pg}

To develop a policy gradient approach for health-based multi-agent systems, 
we begin with the policy gradient theorem~\cite{sutton2018reinforcement_chp13}
\begin{equation}\label{eqn:policy_gradient_theorem}
\nabla J ( \vect{\theta} ) \propto \sum_{\vect{s}} \mu (\vect{s}) \sum_{\vect{u}} q_{\vect{\pi}} \left( \vect{s},\vect{u} \right) \nabla \vect{\pi} \left( \vect{u} \mid \vect{s},\vect{\theta} \right)
\end{equation}
Equation~\ref{eqn:policy_gradient_theorem} gives an analytical expression for the gradient of the objective function with respect to policy parameters $\vect{\theta}$, where $q_{\vect{\pi}}(\vect{s},\vect{u})$ is the true state-action value and $\mu(\vect{s})$ refers to the ergodic state distribution. %
To develop a practical algorithm, we need a method for sampling such an analytical expression that has an expected value equal or approximately equal to Equation \ref{eqn:policy_gradient_theorem}. 
\cref{eqn:policy_gradients_multi_agent} gives a multi-agent form for sampling the policy gradient 
considering local policies $\pi_{\theta_i}$ that map an individual agent's local trajectory $\tau_{i,t}$ to actions $a_{i,t}$:
\begin{equation}\label{eqn:policy_gradients_multi_agent}
g_{\Psi} = \mathbb{E}_{\vect{\pi}} \left[ \sum_{i=1}^{n}  \Psi_{i,t} \nabla_{\theta_i} \log \pi_{\theta_i} (a_{i,t} \mid \tau_{i,t}, \theta_i ) \right]
\end{equation}
where the term $\Psi$ can be formulated to affect the bias and variance of the learning process; see \cref{app:multagent_pg} for derivation~\cite{schulman2015high,foerster2018counterfactual}. %
Typically $\Psi$ takes the form of the return or discounted return $G_t$ (i.e. REINFORCE algorithm~\cite{williams1992simple}); the returns baselined on the state value function $V_{\vect{\pi}}(\vect{s}_t)$ (i.e. REINFORCE with baseline \cite{sutton2000policy}); the state-action value function $Q_{\vect{\pi}}(\vect{s}_t, \vect{u}_t)$; the advantage function $A_{\vect{\pi}}(\vect{s}_t, \vect{u}_t) =  Q_{\vect{\pi}}(\vect{s}_t, \vect{u}_t) - V_{\vect{\pi}}(\vect{s}_t)$; or the generalized advantage function $A_{\vect{\pi}}^{\text{GAE}}$~\cite{schulman2015high,schulman2017proximal}. %
If all agents have the same policy parameters, receive identical rewards throughout an episode, and employ the same $\Psi$ function, then Equation~\ref{eqn:policy_gradients_multi_agent} renders the same gradient at each time step for all agents' actions. %
This uniformity in policy gradients is problematic because it 
results in all actions at a given time step being promoted equally during the next training cycle---i.e. the multi-agent credit assignment problem.  

To overcome this problem, we use the term $\Psi_{i,t}$ to distinguish gradients between different agents' actions at the same time step. %
One option is to set $\Psi_{i,t} = Q_{\pi} \left( \tau_{i,t}, a_{i,t} \right)$, which defines an \emph{observation-action function} referred to as a \emph{local critic}. %
Since all agents are expected to receive distinct observations at each time step, $Q_{\pi} \left( \tau_{i,t}, a_{i,t} \right)$ is expected to be distinct for each agent at a given time step. 
However, this approach relies on making value approximations based on limited information, which can significantly impact learning, as we demonstrate in Section~\ref{sec:experiments}. %
If we instead leverage our prior assumption of \emph{centralized learning}~\cite{gupta2017cooperative,lowe2017multi,foerster2018counterfactual}, then we can make direct use of the central value functions, $V_{\vect{\pi}}(\vect{s}_t)$ and $Q_{\vect{\pi}}(\vect{s}_t, \vect{u}_t)$; however, this still does not resolve the multi-agent credit assignment problem.

We can use the concept of system health 
in an attempt to address the credit assignment problem. %
Our technique stems from the concepts of counterfactual baselines~\cite{foerster2018counterfactual} or difference rewards~\cite{nguyen2018credit}, and is further motivated by the Wonderful Life Utility (WLU)~\cite{wolpert2002optimal}. 
The idea is that credit is assigned to an agent at a given time step by comparing the true joint return from time $t$, 
with the expected return from a hypothetical or ``counterfactual'' scenario where agent $i$ had been terminated at time $t$. 
We propose the following \emph{minimum-health baseline} for multi-agent credit assignment
\begin{equation}\label{eqn:min_health_baseline}
\begin{split}
\Psi_{i,t} &= h_{i,t} \left( G_t - V_{\vect{\pi}} \left( \vect{s}_t^{\neg h_{i,t}}, h_{i,\text{min}}\right)  \right) \\
&= h_{i,t} G_t - b\left(\vect{s}^{\neg i}\right)
\end{split}
\end{equation}
where $h_{i,t}$ is the health of agent $i$ at time $t$ and $\vect{s}_t^{\neg h_{i,t}}$ represents the true joint state of the system at time $t$, except that the health of agent $i$ is replaced with minimum health (typically 0). 

To understand the multiplication by the true health $h_{i,t}$ in Equation~\ref{eqn:min_health_baseline}, consider the implications of Properties~\ref{prop:state_constriction} and~\ref{prop:action_constriction} in Section~\ref{subsec:system_health}. %
If reduction in system health constricts the available actions and reachable states at a given state, and these constrictions are not encoded within the policy $\pi_i$, then an inconsistency arises between chosen and executed actions that can affect training \cite{fujita2018clipped}. %
This would occur if the policy selects an action from a health state that the physical agent is incapable of executing. %

To overcome this inconsistency, we modify the policy gradient such that the contribution from agent $i$ at time $t$ is multiplied by the true health of agent $i$ at time $t$; i.e. $h_{i,t} \in [0,1]$. %
The motivation for this modification is that, as an agent's health degrades and the available action set is constricted, it becomes less likely that the action selected by the policy aligns with the action executed by the agent. %
By attenuating the policy gradient by the health of the agent, policy learning occurs more slowly on data generated in low health states and thus reduces the effect of mismatching chosen and executed actions. %
Section~\ref{sec:experiments} investigates the case where the health state is binary for each agent and a zero-health state shrinks the action set until only a zero-vector action is executable by the agent. %
\cref{lem:non_bias_baseline} states that gradient of the minimum-health baseline term in \cref{eqn:min_health_baseline} is zero, and thus does not introduce bias to the gradient estimate in \cref{eqn:policy_gradients_multi_agent}. %
\cref{lem:binary_health_convergence} states  that health-informed policy gradients converge to a locally optimal policy for systems with binary health states. 

\begin{lemma}\label{lem:non_bias_baseline}
Let the health-informed gradient in \cref{eqn:policy_gradients_multi_agent} be written as $g_{\Psi} = g_h - g_b$, then we have
\begin{equation*}
g_b = \mathbb{E}_{\vect{\pi}} \left[ \sum_{i=1}^{n}  b\left(\vect{s}^{\neg i}\right) \nabla_{\theta_i} \log \pi_{\theta_i} (a_{i,t} \mid \tau_{i,t}, \theta_i ) \right] = 0
\end{equation*}
\end{lemma}

\begin{proof}
Applying \cref{eqn:min_health_baseline} to \cref{eqn:policy_gradient_theorem} we can write
\begin{align*}
g_{\Psi} &= &&\sum_{\vect{s}}{\mu (\vect{s})} \sum_{\vect{u}}{\vect{\pi}(\vect{u} \mid \vect{s}, \vect{\theta})} \cdot \\
& &&\sum_{i=1}^{n} \left(h_i q_{\vect{\pi}}(\vect{s}, \vect{u}) -  b\left(\vect{s}^{\neg i}\right)\right)\nabla_{\theta_i} \log{\pi_{i}(a_{i} \mid \tau_{i}, \theta_{i})} \\
&= &&g_h - g_b
\end{align*}
Separating out the baseline's gradient contribution and applying \cref{eqn:joint_policy_prop}
\begin{align*}
g_b &=&&\sum_{\vect{s}}{\mu (\vect{s})} \sum_{\vect{u}}{\vect{\pi}(\vect{u} \mid \vect{s}, \vect{\theta})} \cdot \\
&&&\sum_{i=1}^{n} b\left(\vect{s}^{\neg i}\right) \nabla_{\theta_i} \log{\pi_{i}(a_{i} \mid \tau_{i}, \theta_{i})} \\
&=&&\sum_{\vect{s}}{\mu (\vect{s})} \sum_{i=1}^{n} b\left(\vect{s}^{\neg i}\right) \cdot \\
&&&\sum_{\vect{u}^{\neg i}} \sum_{a_i} \prod_{j=1, j \neq i}^{n}\pi_j(a_j \mid \tau_j, \theta_j) \nabla_{\theta_i} \pi_{i}(a_{i} \mid \tau_{i}, \theta_{i}) \\
&=&&\sum_{\vect{s}}{\mu (\vect{s})} \sum_{i=1}^{n} b\left(\vect{s}^{\neg i}\right) \cdot \\
&&&\sum_{\vect{u}^{\neg i}} \prod_{j=1, j \neq i}^{n}\pi_j(a_j \mid \tau_j, \theta_j) \nabla_{\vect{\theta}} \sum_{a_i} \pi_{i}(a_{i} \mid \tau_{i}, \theta_{i}) \\
&=&&\sum_{\vect{s}}{\mu (\vect{s})} \sum_{i=1}^{n} b\left(\vect{s}^{\neg i}\right) \cdot \\
&&&\sum_{\vect{u}^{\neg i}} \prod_{j=1, j \neq i}^{n}\pi_j(a_j \mid \tau_j, \theta_j) \nabla_{\vect{\theta}}1 \\
&=&& 0
\end{align*}
\end{proof}

\begin{lemma}\label{lem:binary_health_convergence}
Given a system with binary health states $h_i \in \lbrace 0,1 \rbrace$, where $h_i = 0$ constricts the action space to a single element, $\lvert \mathcal{A}_i \left(\vect{s}_{h_i=0}\right) \rvert = 1$, following the gradient $g_{\Psi}$ in \cref{eqn:policy_gradients_multi_agent} at each iteration $k$ gives 
\begin{equation}
\lim_{k \rightarrow \infty} \| \nabla_{\vect{\theta}} J \| = 0
\end{equation}
\end{lemma}
\begin{proof}
From \cref{lem:non_bias_baseline} and \cref{eqn:policy_gradients_multi_agent}, we have
\begin{align*}
g_{\Psi} &= g_h = \mathbb{E}_{\vect{\pi}} \left[ \sum_{i=1}^{n}  h_{i,t} G_t \nabla_{\theta_i} \log \pi_{\theta_i} (a_{i,t} \mid \tau_{i,t}, \theta_i ) \right].
\end{align*}
With binary health states, we can separate the $n$ agents into groups based on their health at time $t$ such that $\mathcal{I}_1 = \lbrace i : h_{i,t}=1 \rbrace$ and $\mathcal{I}_0 = \lbrace i : h_{i,t}=0 \rbrace$.
For the group of agents in $\mathcal{I}_0$, the action space is constrained such that a single action is deterministically selected, regardless of $\theta_i$. Therefore, 
\begin{align*}
\nabla_{\theta_i} \log \pi_{\theta_i} (a_{i,t} \mid \tau_{i,t}, \theta_i ) = 0, \quad i \in \mathcal{I}_0. 
\end{align*}
As a result, we can say for all $i$
\begin{align*}
h_{i,t} \nabla_{\theta_i} \log \pi_{\theta_i} (a_{i,t} \mid \tau_{i,t}, \theta_i ) = \nabla_{\theta_i} \log \pi_{\theta_i} (a_{i,t} \mid \tau_{i,t}, \theta_i ).
\end{align*}
Applying this result and \cref{eqn:joint_policy_prop} gives
\begin{align*}
g_{\Psi} &= \mathbb{E}_{\vect{\pi}} \left[ G_t \sum_{i}^{n}  \nabla_{\theta_i} \log \pi_{\theta_i} (a_{i,t} \mid \tau_{i,t}, \theta_i ) \right] \\
&=  \mathbb{E}_{\vect{\pi}} \left[ G_t \nabla_{\vect{\theta}} \log \vect{\pi}_{\vect{\theta}} (\vect{u}_t \mid \vect{s}_t, \vect{\theta} ) \right],
\end{align*}
which is the gradient of the single-agent REINFORCE algorithm that has proven convergence properties for undiscounted returns \cite{sutton2000policy,williams1992simple}.
\end{proof}

Equation~\ref{eqn:min_health_baseline} provides a health-informed counterfactual baseline on the state value function that is completely agnostic to the action space, a property not seen in prior work~\cite{foerster2018counterfactual,wolpert2002optimal,nguyen2018credit,
rashid2018qmix,mahajan2019maven}. %
This makes our approach well suited for use in policy optimization techniques for large or continuous action spaces, such as TRPO and PPO~\cite{schulman2015trust,schulman2017proximal}.

Other recent work on multi-agent credit assignment draws inspiration from WLU and Aristocrat Utility (AU).
Nguyen et al. \cite{nguyen2018credit} offer a modern form of WLU that requires maintaining explicit counts of discrete actions and state visitations and is not well posed for continuous domains and deep-RL. %
Modern implementations of Aristocrat Utility (AU) such as COMA~\cite{foerster2018counterfactual} require enumeration over all possible actions or computationally expensive Monte Carlo analysis at each time step. %
The health-informed baseline suffers no such limitations. %

\subsection{Health-Informed Multi-Agent Proximal Policy Optimization}\label{subsec:health_marl_ppo}

While the health-informed credit assignment technique described in Section~\ref{subsec:health_marl_pg} is applicable to any reinforcement learning algorithm that uses value functions, we choose to demonstrate our technique within a multi-agent variant of proximal policy optimization (PPO) 
\cite{schulman2017proximal}. 
PPO is chosen because it has been shown to work well with continuous action spaces~\cite{duan2016benchmarking} as well as multi-agent environments~\cite{openai2019five}.

Our health-informed counterfactual baseline is evaluated using a centralized critic, $V_w(\vect{s}_t)$, which we model with a deep neural network parameterized by weights $w$. As with the original PPO paper \cite{schulman2017proximal}, the centralized critic is trained using generalized advantage estimation \cite{schulman2015high} where centralized value targets are defined as 
\begin{equation}\label{eqn:value_target}
V_t^\text{targ} = A_t^\text{GAE} + V_{w_\text{old}}(\vect{s}_t),
\end{equation}
and the centralized value loss function is
\begin{equation}\label{eqn:value_update}
L^{\text{VF}}(w) = \left( V_w(\vect{s}_t) - V_t^\text{targ} \right)^2.
\end{equation}
Now replacing the returns, $G_t$, in \cref{eqn:min_health_baseline} with the more general value targets, our counterfactual baseline becomes
\begin{equation}\label{eqn:min_health_baseline_ppo}
\Psi_{i,t} = h_{i,t} \left( V_t^\text{targ} - V_{w_\text{old}} \left( \vect{s}_t^{\neg i}\right)  \right)
\end{equation}
We apply the health-informed counterfactual baseline in \cref{eqn:min_health_baseline_ppo} to PPO's surrogate objective function and formulate the clipped surrogate objective function 
\begin{equation}\label{eqn:surrogate_objective}
L (\theta) = \mathbb{E}_t \left[ \frac{\pi_{\theta}(a_{i,t} \mid \tau_{i,t})}{\pi_{\theta_\text{old}}(a_{i,t} \mid \tau_{i,t})} \, \Psi_{i,t}  \right] = \mathbb{E}_t \left[ \rho_{i,t}(\theta) \Psi_{i,t} \right]
\end{equation}
\begin{equation}\label{eqn:surrogate_objective_clipped}
\begin{split}
L^{\text{CLIP}+\text{S}} (\theta)  = \mathbb{E}_t [ c S_{\pi_\theta} & + \\
\text{min} ( &\rho_{i,t}(\theta) \Psi_{i,t}, \\
& \text{clip} \left( \rho_{i,t}(\theta), 1 - \epsilon, 1 + \epsilon \right) \Psi_{i,t} )].
\end{split}
\end{equation}
As in the original PPO paper, we also need to train the value function network and augment with an entropy bonus---$S_{\pi_\theta}$ weighted by hyperparameter $c$---to encourage exploration \cite{schulman2017proximal}. %

\cref{alg:mappo_baseline} gives the training process for a group of $n$ cooperative agents in a centralized-learning, decentralized-execution framework using the health-informed multi-agent proximal policy optimization (MAPPO) algorithm and parameter sharing.

\begin{algorithm}[tb]
   \caption{Health-Informed Multi-Agent PPO}
   \label{alg:mappo_baseline}
\begin{algorithmic}
   \STATE Initialize $\theta$ and $w$
   \FOR{$iteration=1,2,...$}
   \STATE Run local policies $\pi_{\theta}$ on $n$ agents for $t_f$ timesteps
   \STATE Compute value targets $V_t^\text{targ}, \forall t \in \{1,..,t_f \}$
   \STATE Compute $\Psi_{i,t}$ with \cref{eqn:min_health_baseline_ppo} $\forall i, \forall t$
   \STATE Compute $\theta'$ with PPO update in \cref{eqn:surrogate_objective_clipped}, with $K$ epochs and minibatch size $M \leq n t_f$
   \STATE Compute $w'$ with~\cref{eqn:value_update}
   \STATE Update $\theta \leftarrow \theta'$, $w \leftarrow w'$
   \ENDFOR
\end{algorithmic}
\end{algorithm}

\section{Experiments}\label{sec:experiments}

\begin{figure*}
\setlength\tabcolsep{1.5pt}
\centering
	\begin{tabular}{ccc}
\subfloat[]{
	\includegraphics[width=0.3\textwidth]{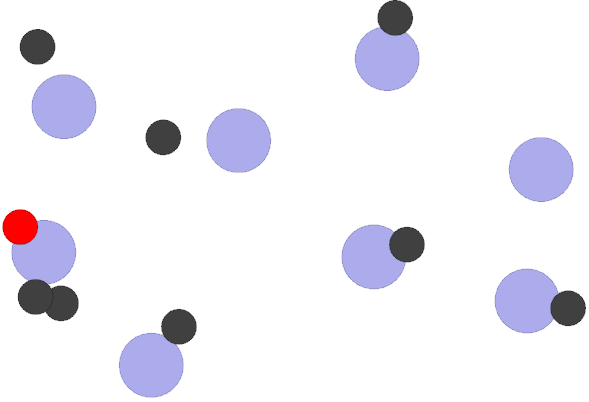}
	\label{fig:hazard_nav}
	} &
	\subfloat[]{
	\includegraphics[width=0.3\textwidth]{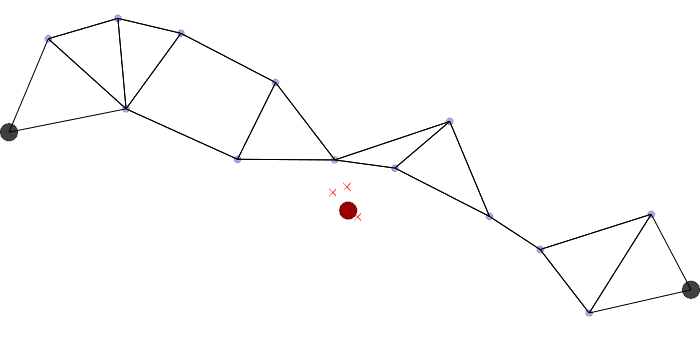}
	\label{fig:hazard_comm}
} &
	\subfloat[]{
	\includegraphics[width=0.3\textwidth]{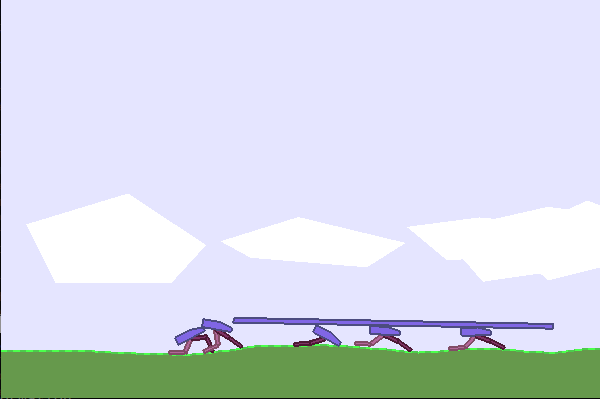}
	\label{fig:multiwalker}
}
	\end{tabular}
	\caption{The 8-agent hazard navigation (\ref{fig:hazard_nav}), 16-agent hazardous communication (\ref{fig:hazard_comm}), and 5-agent multiwalker (\ref{fig:multiwalker}) environments. 
	(\ref{fig:hazard_nav}) black dots represent landmarks to be covered by agents. The red dot is the hazardous landmark that has been `revealed' by an agent in its vicinity. 
	(\ref{fig:hazard_comm}) The larger black dots represent the terminals to be connected. 
The smaller dots represent agents. Links are formed between agents that are within each other's communication radius. The red dot is the environmental hazard. Tiny red crosses represent agents that have been terminated due to proximity to the hazard.
	(\ref{fig:multiwalker}) Walker balance the package on their heads and walk together. Fallen agents are given a zero-health and rendered immobilized}
\label{fig:environments}
\end{figure*}

\begin{figure*}
\setlength\tabcolsep{1.5pt}
\centering
	\begin{tabular}{cc}
	\subfloat[]{
	\includegraphics[width=0.45\textwidth]{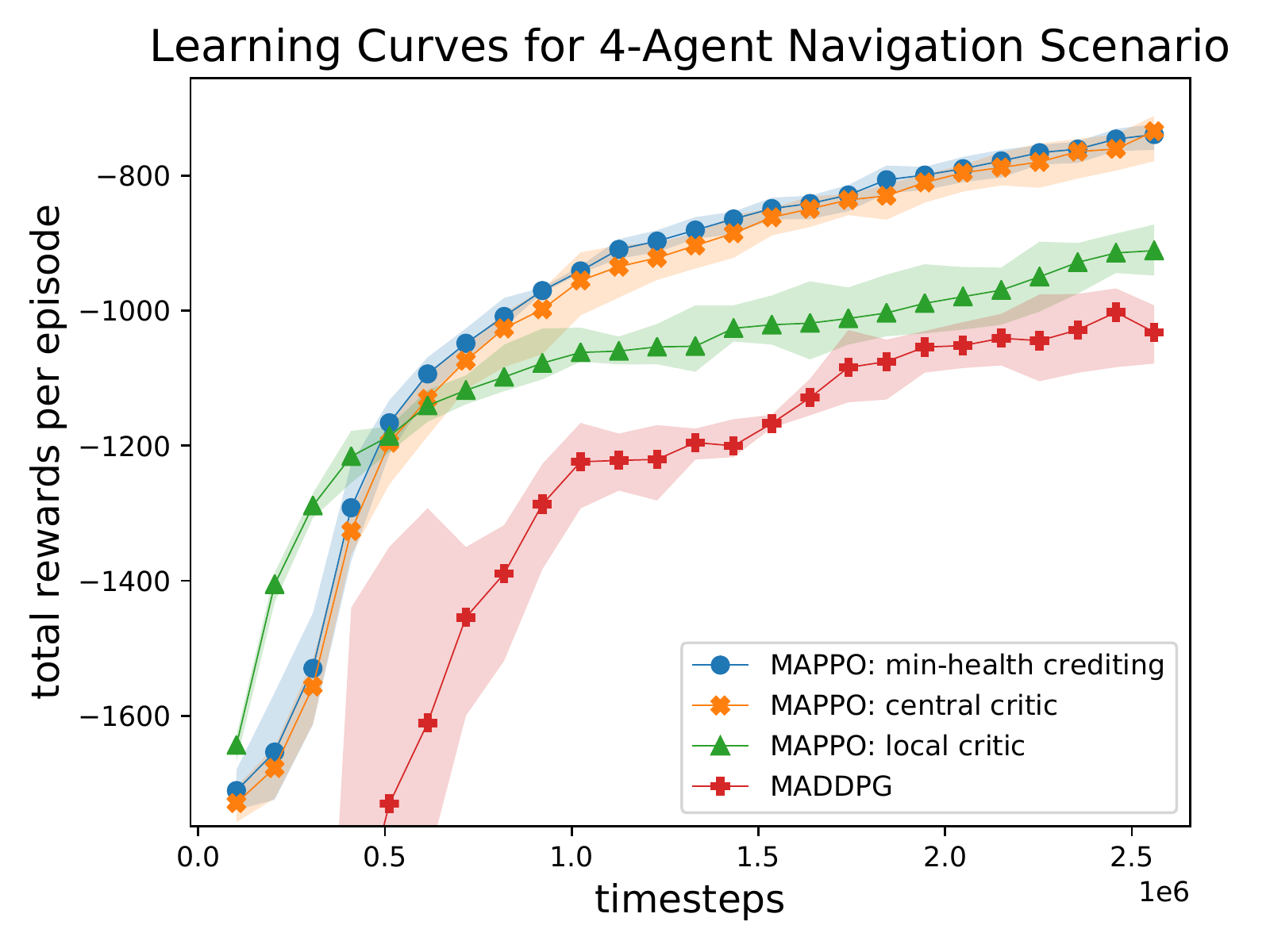}
	\label{fig:4agent_spread}
	} &
	\subfloat[]{
	\includegraphics[width=0.45\textwidth]{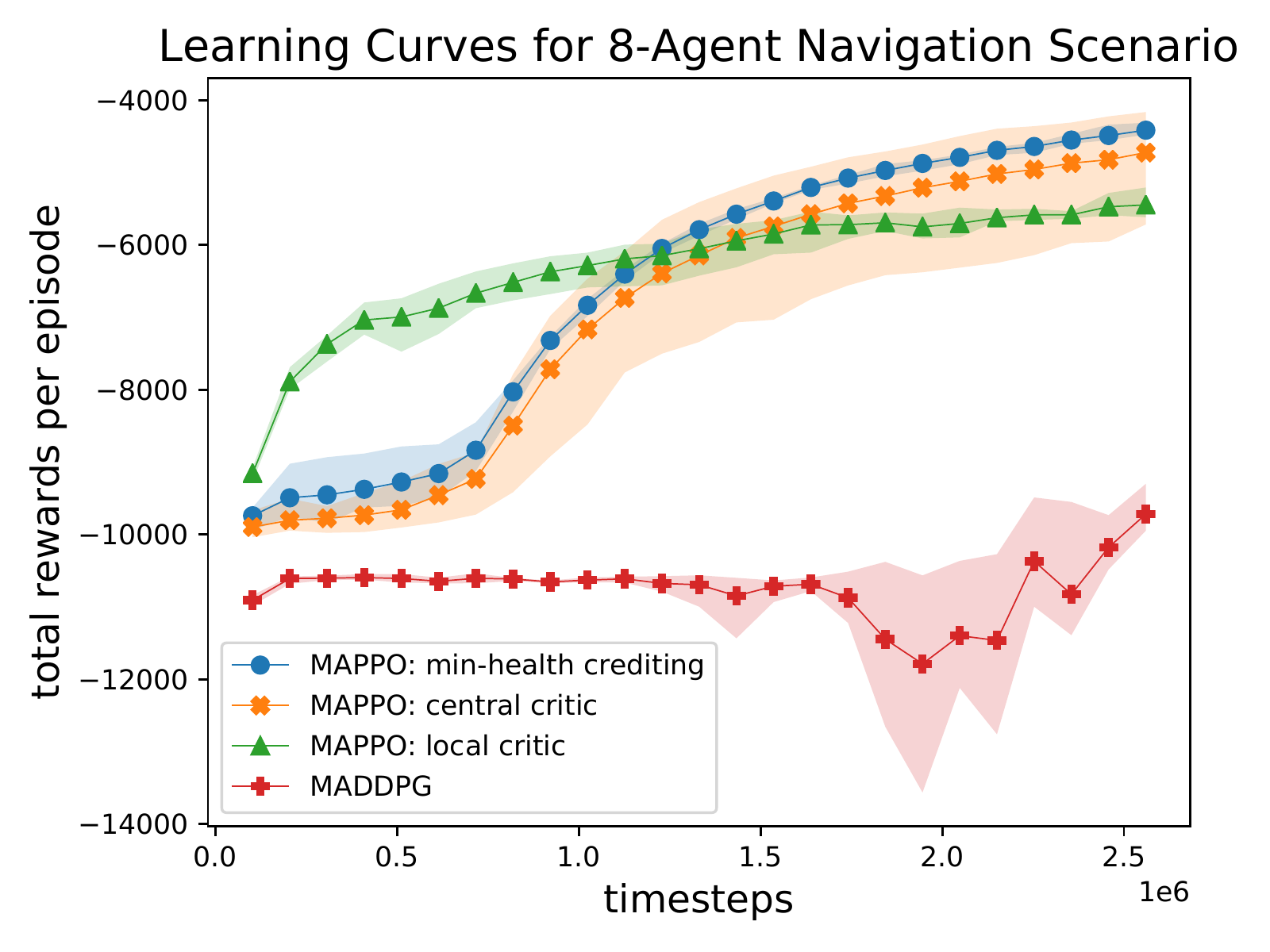}
	\label{fig:8agent_spread}
	} \\
	\subfloat[]{
	\includegraphics[width=0.45\textwidth]{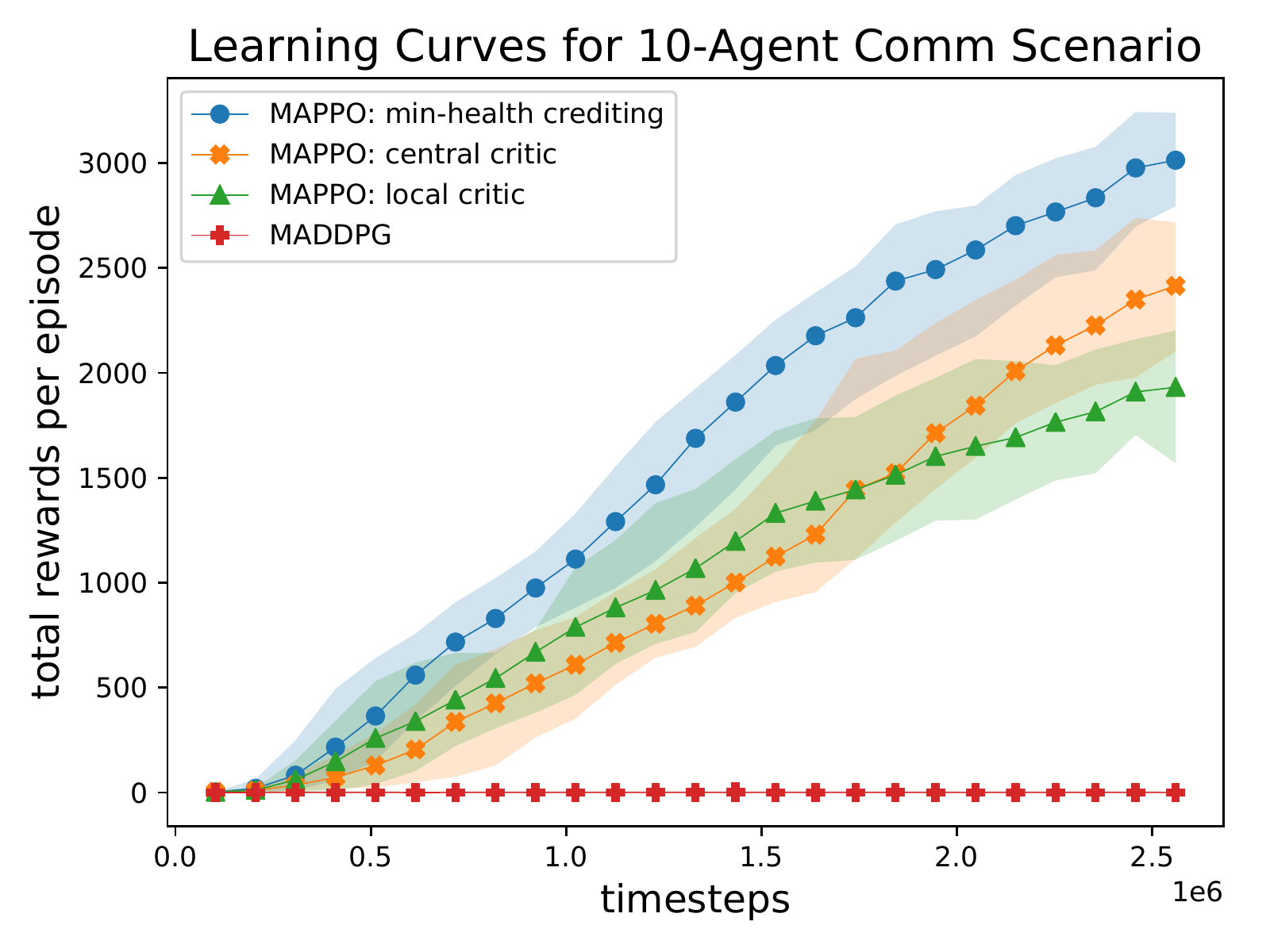}
	\label{fig:10agent_graph}
	} &
	\subfloat[]{
	\includegraphics[width=0.45\textwidth]{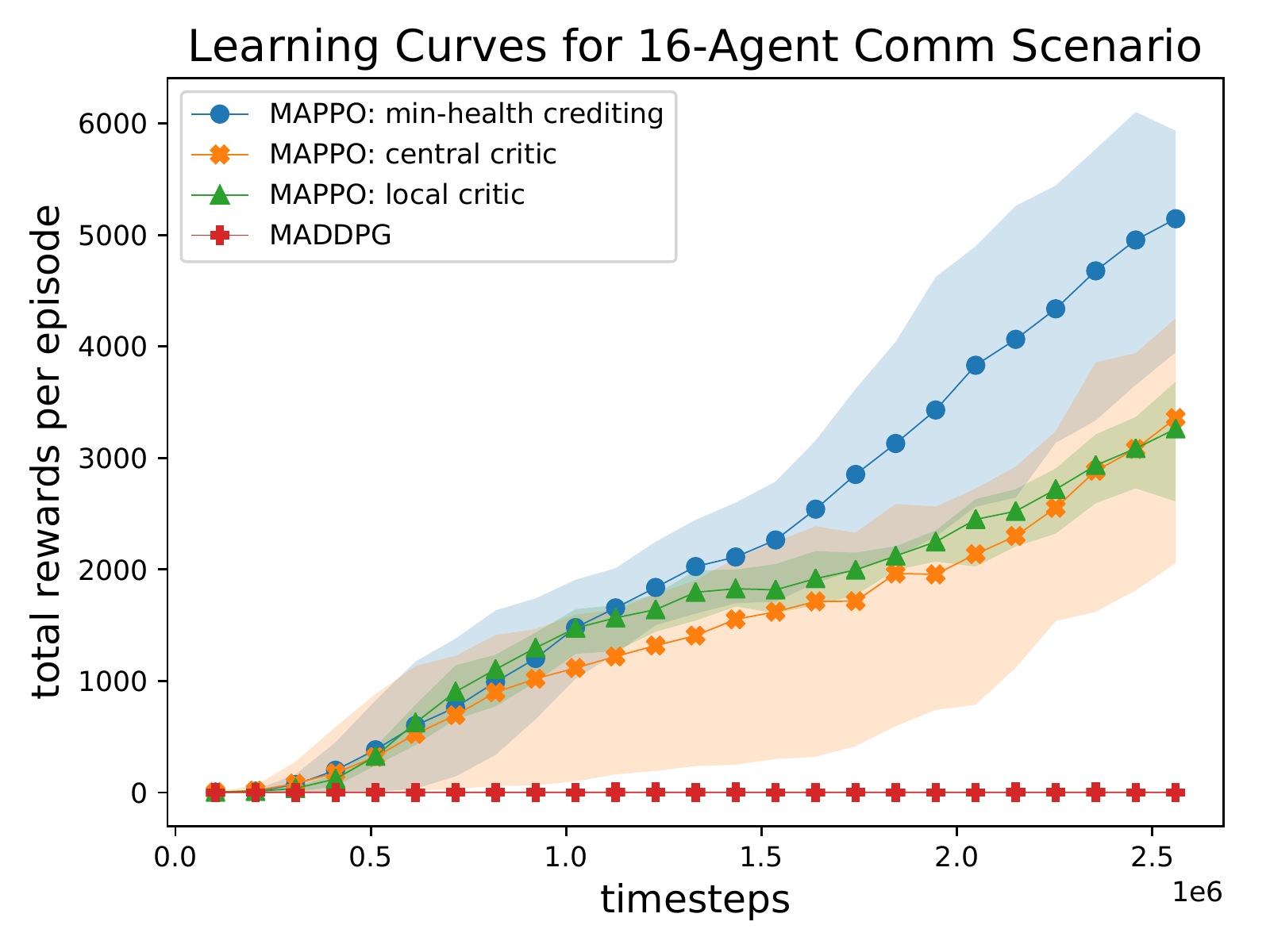}
	\label{fig:16agent_graph}
	} \\
	\subfloat[]{
	\includegraphics[width=0.45\textwidth]{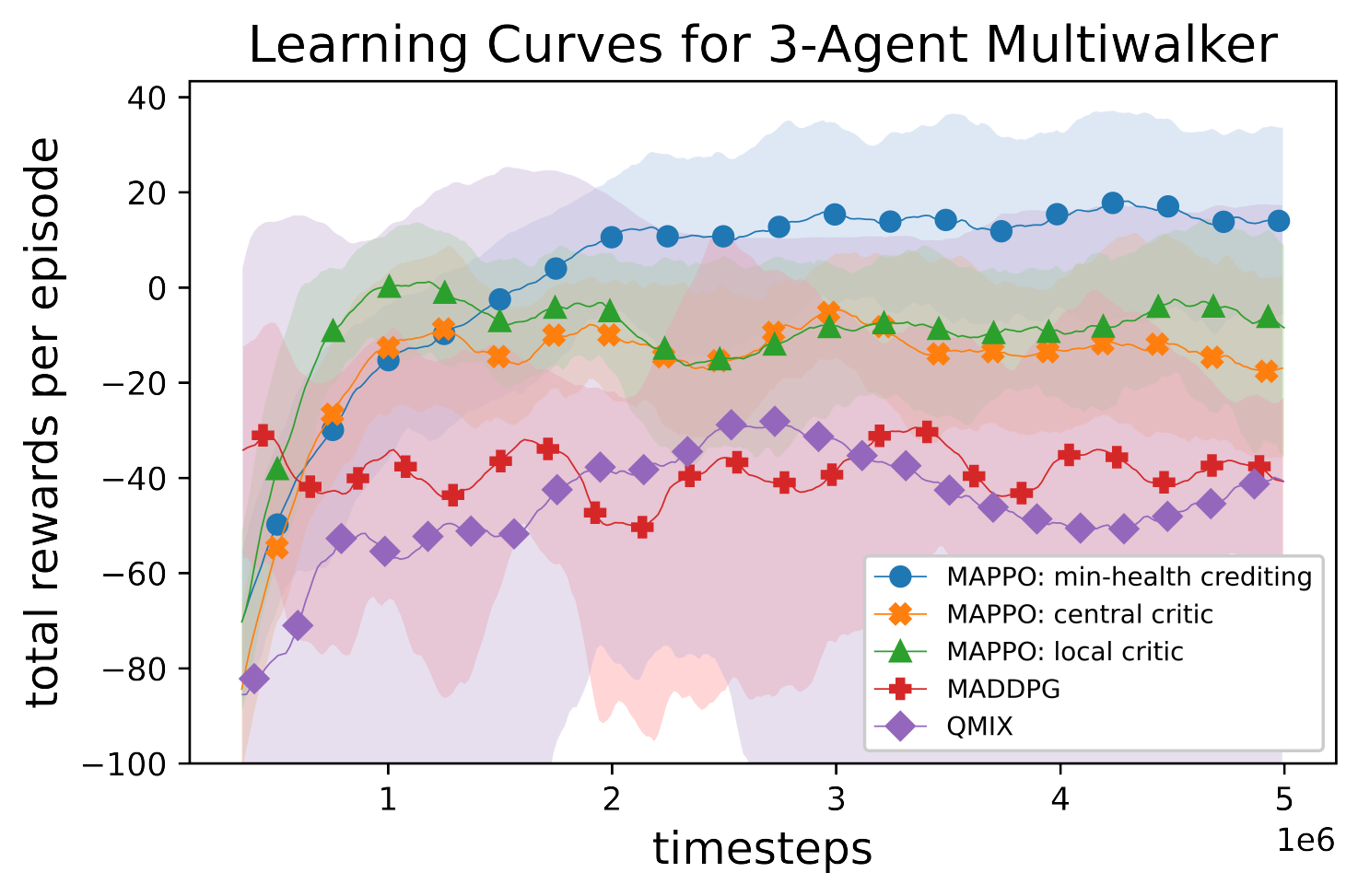}
	\label{fig:3agent_walker}
	} &
	\subfloat[]{
	\includegraphics[width=0.45\textwidth]{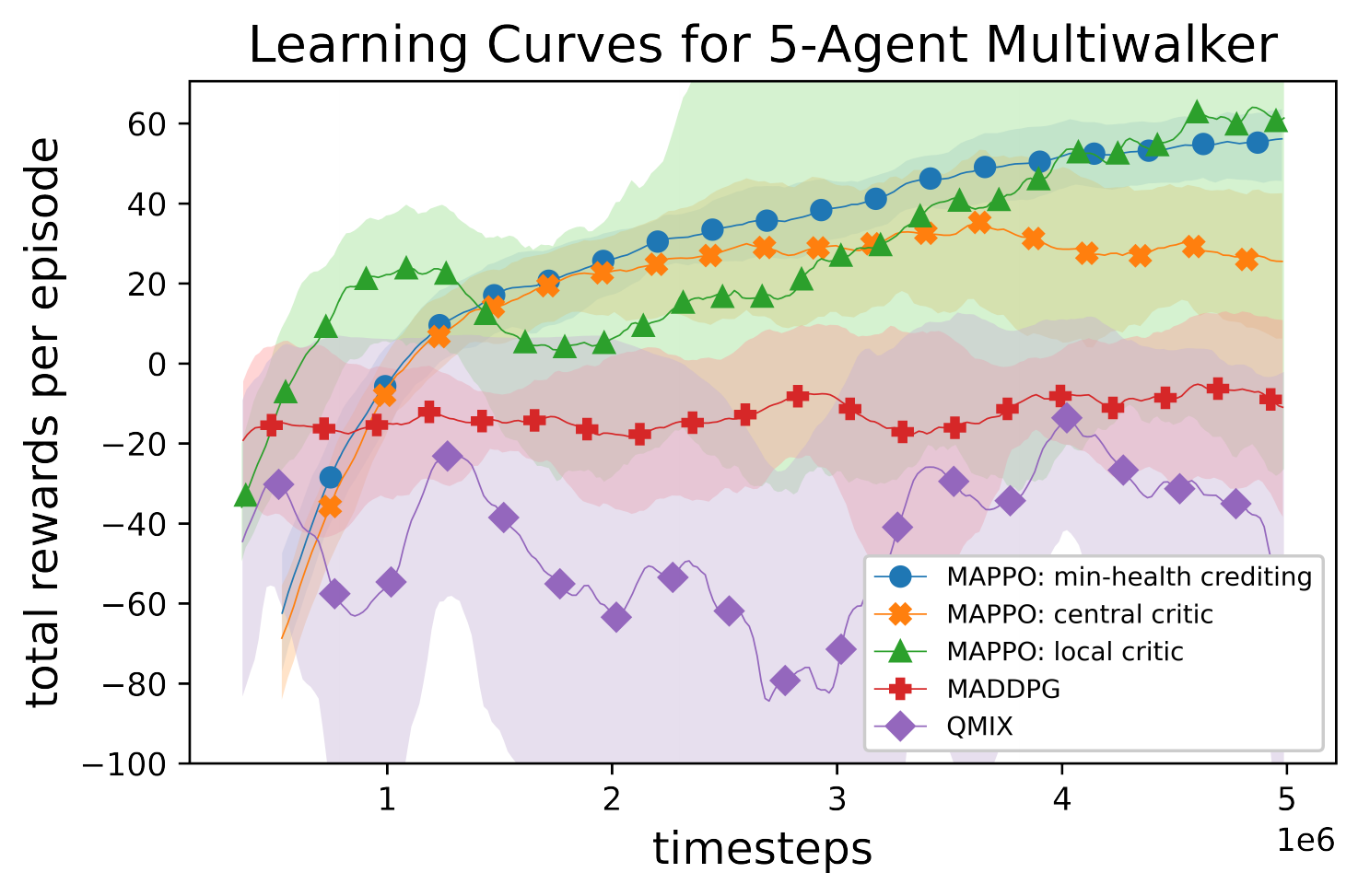}
	\label{fig:5agent_walker}
	}
	\end{tabular}
	\caption{Learning curves for the hazardous navigation (\ref{fig:4agent_spread}, \ref{fig:8agent_spread}),  hazardous communication (\ref{fig:10agent_graph}, \ref{fig:16agent_graph}), and multiwalker (\ref{fig:3agent_walker}, \ref{fig:5agent_walker}) environments. %
	Each campaign defines an environment and the number of agents present in that environment and produces a learning curve for MADDPG, three variants of MAPPO, and QMIX (note: QMIX is only implemented in multiwalker environment where RLlib was used as the training engine). %
	A single learning curve is an average over four independent training experiments with the shaded region representing the minimum and maximum bounds of the four training experiments.
	MADDPG is shown in red crosses. %
	MAPPO with a local critic is shown in green triangles. %
	MAPPO with a non-crediting central critic is shown in orange X's. %
	MAPPO with a central critic and the health-informed, counterfactual baseline given in Equation~\ref{eqn:min_health_baseline} (referred to as \emph{min-health crediting}) is shown in blue circles. 
	QMIX is shown in purple diamonds. 
	The algorithms and environments used to generate these results are publicly available at: 
	\ifarxiv \protect\url{https://github.com/rallen10/ergo_particle_gym}
	\else [WITHHELD TO PRESERVE ANONYMITY DURING REVIEW PROCESS]\fi
	}
	\label{fig:results_table}
\end{figure*}

This section presents experiments to demonstrate health-informed credit assignment in multi-agent learning. %
Two separate implementations of health-informed proximal policy optimization were developed. The first implementation was built with a TensorFlow framework within a forked version of OpenAI's Multi-Agent Particle Environment (MPE) library that is extended to incorporate the concepts of system health and risk-taking 
\cite{lowe2017multi,openai2018multiagent}. %
To take advantage of recently developed multi-agent RL toolsets, a second implementation was built with PyTorch within RLlib \cite{liang2018rllib} and trained in the PettingZoo multiwalker environment \cite{terry2020pettingzoo,gupta2017cooperative}. %

Since the original MPE environments were dedicated to small groups of agents and did not incorporate the concept of health, new scenarios were developed.\footnote{See Appendix~\ref{app:maddpg_benchmark} for a comparison between MADDPG and our proposed variants of MAPPO in an environment taken directly from the original MADDPG paper~\cite{lowe2017multi}.} 
The two MPE scenarios, titled \emph{hazardous navigation} and \emph{hazardous communication network}, and the PettingZoo \emph{multiwalker} environment are described later in this section.

The experiments compare multi-agent deep deterministic policy gradients (MADDPG) and QMIX (for RLlib implementation only) \cite{rashid2018qmix} with three multi-agent variants of proximal policy optimization (MAPPO) referred to as \emph{local critic}, \emph{central critic} and \emph{min-health crediting}. 
The local critic MAPPO uses an advantage function based on local observation value estimates $ \Psi_{i,t} = A_t^\text{GAE}\left(\tau_{i,t_f} \right) $.
The central critic MAPPO uses an advantage function based on joint state value estimates enabled by the centralized learning assumption: $ \Psi_{i,t} = A_t^\text{GAE}\left( \vect{s}_t \right) $.
The minimum-health crediting MAPPO uses the health-informed counterfactual baseline in Equation \ref{eqn:min_health_baseline_ppo}. %

\textbf{Hazardous navigation.} This environment is closely related to Lowe's ``cooperative navigation'' environment where agents must cooperate to reach a set of landmarks~\cite{lowe2017multi}. %
Reward is based on the distance from each landmark to the nearest agent, thus all agents receive the same reward and each landmark should be `covered' by a separate agent in order to maximize reward. %
Our variation of this problem incorporates a hazardous landmark that can probabilistically cause an agent to be terminated (i.e. spontaneously transition to a zero-health state) if the agent is within a threshold distance of the hazard. %
The landmark that poses a hazard is not known until at least one agent crosses its threshold distance. %
To connect this scenario to a real-world problem, consider the use of uninhabited aerial vehicles (UAVs) to monitor wildfires. Each spot fire must be continuously monitored and one spot fire poses significant risk to any UAV within its proximity. 
Figure~\ref{fig:hazard_nav} is a snapshot of the hazardous navigation environment.

\textbf{Hazardous communication network.} This environment consists of two fixed landmark terminals and a group of mobile agents that can serve as communication relays over short distances. %
The objective is for the agents to cooperatively arrange themselves into an uninterrupted chain linking the two terminals. All agents receive the same reward for every time step in which the terminals are connected, and zero reward when the link is broken. 
For each episode an environmental hazard is randomly placed between the terminals, which causes agents in its vicinity to be terminated with some probability $p_{\text{fail}}$. %
Figure~\ref{fig:hazard_comm} shows a snapshot of the hazardous communication network scenario. 

\textbf{Multiwalker. } This environment, originally presented in Gupta et al. \cite{gupta2017cooperative} and adapted into the PettingZoo library \cite{terry2020pettingzoo}, consists of $N$ robots that must collaboratively carry a package. The group of robots is rewarded based on the distance the package has moved. Each agent observes the relative pose of neighboring walkers as well as the pose of the package. As control input, each agent selects the joint torques to apply to their leg joints. The agents must learn to walk as well as carry the package in order to achieve high reward. Unlike the original implementation of multiwalker that used local reward shaping to avoid the multi-agent credit assignment problem \cite{gupta2017cooperative}, this version of the problem enforces a joint reward signal. If an agent falls to the ground, it's health is set two zero and can take no further action. Figure \ref{fig:multiwalker} gives a snapshot of the multiwalker environment.

All three environments are characterized by partially observed state spaces, continuous action spaces, joint rewards, and agent attrition; making them particularly challenging for multi-agent reinforcement learning.

\subsection{Results}

Figure~\ref{fig:results_table} summarizes the results of our experiments. 
A few trends that are common to all campaigns immediately emerge. %
Most notably, we see that MAPPO with health-informed crediting tends to outperform all other multi-agent RL algorithms for each of the environments and group sizes tested. %
Furthermore, by comparing \ref{fig:4agent_spread}-to-\ref{fig:8agent_spread} and \ref{fig:10agent_graph}-to-\ref{fig:16agent_graph}, we see that the performance gap between health-informed crediting and other algorithms tends to increase as the number of agents increases. %
We expect such a trend because, as the number of agents in the environment increases, the more pronounced the multi-agent credit assignment problem becomes, and therefore crediting algorithms such as Equation~\ref{eqn:min_health_baseline} should show increasing benefits. %
In cases where health-informed crediting only slightly outperforms a non-crediting approach (such as \ref{fig:4agent_spread}, \ref{fig:8agent_spread} , and \ref{fig:5agent_walker}), we see that the health-informed baseline has the added benefit of reducing variance between trials within a training campaign. %
This can be seen by comparing the relative sizes of the blue shaded regions with the orange shaded regions.

In general, centralized critics---which include MAPPO: central critic and MAPPO: min-health crediting---tend to outperform local critics and always outperform MADDPG and QMIX for the environments tested. %
MADDPG's underperformance is likely due to the fact that these environments consist of multi-agent group sizes considerably larger than those developed in the original MADDPG work. %
This would explain why MADDPG shows its best performance in Figure~\ref{fig:4agent_spread} which is the scenario most closely aligned with the original ``cooperative navigation'' environment~\cite{lowe2017multi}. %

The poor perofmance of QMIX is almost certainly due to the fact the algorithm is fundamentally designed for discrete action spaces but our environments all consist of continuous action spaces. In order to run QMIX, the action space was discretized into eight action bins per joint, but this did not give sufficient resolution to enable QMIX to learn a suitable policy. This highlights the advantage of credit assignment with a policy optimization algorithm (i.e PPO) in contrast to value-based methods like QMIX and COMA.

A local critic tends to learn more quickly, outperforming other algorithms in the short term, but then plateaus and is overtaken by central critic approaches. 
An interesting exception to this trend is the 5-Agent multiwalker experiment (\ref{fig:5agent_walker}) where local-critic MAPPO appears to perform on par with the health-informed MAPPO. %
The local critic MAPPO aggregate performance is heavily influenced by a trial that seemed to discover an exploit in the environment that produced high rewards with little coordination between the walkers; the agents learned a shaking/jerking motion---instead of a walking motion---that slid the package forward like an object moving on a shaking table. The local critic algorithm was not able to reproduce the behavior in other training, thus the wide variance in experiments.

Experiments were run on Intel Xeon E5-2687W CPUs in a 32-core Linux desktop. The particle environment experiment were run for 50,000 episodes with each episode consisting of 50 time steps and training batches composed of 256 episodes. Training batches were broken into 8 minibatches and run over 8 epochs. %
Multiwalker trained for 5 million timesteps, with training batches of 16384 timesteps, minibatches 4096 timesteps, and 32 epochs. For the multi-agent PPO experiments an entropy coefficient of 0.01 was used for particle environments to ensure sufficient exploration \cite{schulman2017proximal} and 0.0 in multiwalker to stabilize training. 

For all experiments represented in Figure~\ref{fig:results_table} the policy network was composed of a multilayer perceptron (MLP) with 2 fully connected hidden layers, each of which being 64 units wide, and a hyberbolic tangent activation function. %
For experiments that utilized a local critic the value function network matched the architecture of the policy network. For experiments that utilized a centralized critic the value function network had a distinct architecture that was developed empirically. %
Such centralized critic networks consisted of a 8-layer by 64-unit, fully connected MLP that used an exponential linear unit (ELU) activation function~\cite{clevert2015fast}. %
We observe that the ELU activation function tended to outperform the rectified linear unit (ReLU) and hyperbolic tangent activation functions for central critic learning. 

For the particle environments an actor learning rate of \num{1.0e-3} and a central critic learning rate of \num{5.0e-3} was used with the Adam optimizer~\cite{kingma2014adam}. For the multiwalker environment run with RLlib, a learning rate of \num{3e-4} was used.

\section{Conclusions}

In this paper we have proposed a definition for system health and shown how it can be used in policy gradient methods to improve multi-agent reinforcement learning in a certain class of Dec-POMDPs. %
The techniques presented here are well suited for solving continuous-control multi-robot coordination problems in hazardous environments such as search and rescue (e.g. exploring burning or collapsing buildings) disaster relief (e.g. mapping wildfires or toxic chemical leaks with groups of UAVs) and coordinated load carrying. %
These techniques are also well suited for reinforcement learning in multi-agent adversarial game environments that exhibit agent attrition, such as StarCraft II and DOTA 2 \cite{vinyals2019grandmaster,openai2019five}. %

This work raises several questions that merit future investigation. %
The logical next step would be to explore whether a similar form of counterfactual reasoning can help address multi-agent credit assignment in continuous-control environments that \emph{are not} characterized by system health. This perhaps could be achieved by generating training data in environments with fewer agents than the target environment, but it is uncertain what side effects this type of ``off-environment" experience would have on training.
Additionally, the combination of the action space constriction in \cref{prop:action_constriction} and health-informed policy gradient in \cref{eqn:policy_gradients_multi_agent} and \cref{eqn:min_health_baseline} highlight the need for further investigation of policy gradients on systems where actions chosen by the policy do not exactly match the actions executed by an agent or agents \cite{fujita2018clipped}.

\bibliographystyle{ACM-Reference-Format}
\bibliography{../bibfiles/master.bib}  

\clearpage

%
%
%


\appendix

\section{Extended Proofs for Multi-Agent Policy Gradients}\label{app:multagent_pg}

%

Here elaborate on the equations presented in \cref{subsec:health_marl_pg}. 
Using Equation~\ref{eqn:joint_policy_prop}, we can derive a multi-agent policy gradient---without health-informed baseline---from the single-agent policy gradient. %
We then use this result to provide further detail on the intermediate steps in \cref{lem:non_bias_baseline}.

\begin{lemma}\label{lem:multi_agent_policy_gradient}
For $n$ independent agents per \cref{eqn:joint_policy_prop}
\begin{equation*}
\nabla J ( \vect{\theta} ) \propto g = \mathbb{E}_{\vect{\pi}} \left[ \sum_{i=1}^{n}  G_t \nabla_{\theta_i} \log \pi_{\theta_i} (a_{i,t} \mid \tau_{i,t}, \theta_i ) \right]
\end{equation*}
\end{lemma}

\begin{proof}
From Sutton and Barto \cite{sutton2018reinforcement_chp13} and \cref{eqn:joint_policy_prop} we have:
\begin{align*}
g =& \sum_{\vect{s}}{\mu (\vect{s})}\sum_{\vect{u}}{q_{\vect{\pi}}(\vect{s}, \vect{u}) \nabla_{\vect{\theta}}\vect{\pi}(\vect{u} \mid \vect{s}, \vect{\theta})} \\
=& \sum_{\vect{s}}{\mu (\vect{s})}\sum_{\vect{u}}{q_{\vect{\pi}}(\vect{s}, \vect{u}) \vect{\pi}(\vect{u} \mid \vect{s}, \vect{\theta})} \cdot \\
& \qquad \nabla_{\vect{\theta}}\log{\vect{\pi}(\vect{u} \mid \vect{s}, \vect{\theta})} \\
=& \sum_{\vect{s}}{\mu (\vect{s})}\sum_{\vect{u}}q_{\vect{\pi}}(\vect{s}, \vect{u}) \vect{\pi}(\vect{u} \mid \vect{s}, \vect{\theta}) \cdot \\
& \qquad \nabla_{\vect{\theta}}\log{\prod_{i=1}^{n} \pi_{i}(a_{i} \mid \tau_{i}, \theta_i)} \\
=& \sum_{\vect{s}}{\mu (\vect{s})}\sum_{\vect{u}}{q_{\vect{\pi}}(\vect{s}, \vect{u}) \vect{\pi}(\vect{u} \mid \vect{s}, \vect{\theta})} \cdot \\
& \qquad \nabla_{\vect{\theta}}\sum_{i=1}^{n}\log{\pi_{i}(a_{i} \mid \tau_{i}, \theta_i)} \\
=& \sum_{\vect{s}}{\mu (\vect{s})}\sum_{\vect{u}}{q_{\vect{\pi}}(\vect{s}, \vect{u}) \vect{\pi}(\vect{u} \mid \vect{s}, \vect{\theta})} \cdot \\
& \qquad \sum_{i=1}^{n} \nabla_{\theta_i} \log{\pi_{i}(a_{i} \mid \tau_{i}, \theta_i)},
\end{align*}
which gives the intermediate result
\begin{multline}\label{eqn:multiagent_pg_1}
g = \sum_{\vect{s}}{\mu (\vect{s})}\sum_{\vect{u}}{\vect{\pi}(\vect{u} \mid \vect{s}, \vect{\theta})} \cdot \\
\sum_{i=1}^{n} q_{\vect{\pi}}(\vect{s}, \vect{u}) \nabla_{\theta_i} \log{\pi_{i}(a_{i} \mid \tau_{i}, \theta_i)}.
\end{multline}

Continuing the derivation by replacing sums over variables with expectations of samples drawn following the policy $\vect{\pi}$ \cite{sutton2018reinforcement_chp13}
\begin{align*}
g &= \mathbb{E}_{\vect{\pi}} \left[ \sum_{\vect{u}}{\vect{\pi}(\vect{u} \mid \vect{s}_t, \vect{\theta})} \right. \cdot \\
& \qquad \left. \sum_{i=1}^{n} q_{\vect{\pi}}(\vect{s}_t, \vect{u}) \nabla_{\theta_i} \log{\pi_{i}(a_{i} \mid \tau_{i,t}, \theta_i)} \right] \\
&= \mathbb{E}_{\vect{\pi}} \left[ \sum_{i=1}^{n} q_{\vect{\pi}}(\vect{s}_t, \vect{u}_t) \nabla_{\theta_i} \log{\pi_{i}(a_{i,t} \mid \tau_{i,t}, \theta_{i})} \right].
\end{align*}
Finally arriving at 
\begin{equation}\label{eqn:multiagent_pg_2}
g = \mathbb{E}_{\vect{\pi}} \left[\sum_{i=1}^{n} G_t \nabla_{\theta_i} \log{\pi_{i}(a_{i,t} \mid \tau_{i,t}, \theta_i)} \right]
\end{equation}
\end{proof}

From \cref{eqn:multiagent_pg_2} we can then arrive at \cref{eqn:policy_gradients_multi_agent} by substituting $\Psi_{i,t}$ for $G_t$ in order to  affect the bias, variance, and credit assignment in the policy gradient.

To investigate the convergence of the health-informed policy gradient, let us apply Equation~\ref{eqn:min_health_baseline} in place of $q_{\vect{\pi}}(\vect{s}, \vect{u})$ in the final term in Equation~\ref{eqn:multiagent_pg_1}. We have

\begin{align*}
\begin{split}
g_{\Psi} &= \sum_{\vect{s}}{\mu (\vect{s})} \sum_{\vect{u}}{\vect{\pi}(\vect{u} \mid \vect{s}, \vect{\theta})} \sum_{i=1}^{n} \cdot \\
& \qquad \left( h_i q_{\vect{\pi}}(\vect{s}, \vect{u}) - b\left(\vect{s}^{\neg i}\right)\right) \nabla_{\theta_i} \log{\pi_{i}(a_{i} \mid \tau_{i}, \theta_{i})} \\
&= \sum_{\vect{s}}{\mu (\vect{s})} \sum_{\vect{u}}{\vect{\pi}(\vect{u} \mid \vect{s}, \vect{\theta})} \cdot \left( \vphantom{\sum_{i=1}^{n}} \right.\\
& \qquad \sum_{i=1}^{n} h_i q_{\vect{\pi}}(\vect{s}, \vect{u}) \nabla_{\theta_i} \log{\pi_{i}(a_{i} \mid \tau_{i}, \theta_{i})} + \\
&\qquad -\sum_{i=1}^{n} b\left(\vect{s}^{\neg i}\right) \nabla_{\theta_i} \log{\pi_{i}(a_{i} \mid \tau_{i}, \theta_{i})} \left. \vphantom{\sum_{i=1}^{n}} \right) \\
&= g_h - g_b
\end{split}
\end{align*}

Now we show that $g_b = 0$. Note that this is equivalent to the proof given in \cref{lem:non_bias_baseline} but with more intermediate steps included to help the reader follow. 

\begin{align*}
\begin{split}
g_b &= \sum_{\vect{s}}{\mu (\vect{s})} \sum_{\vect{u}}{\vect{\pi}(\vect{u} \mid \vect{s}, \vect{\theta})} \cdot \\
& \qquad \sum_{i=1}^{n} b\left(\vect{s}^{\neg i}\right) \nabla_{\theta_i} \log{\pi_{i}(a_{i} \mid \tau_{i}, \theta_{i})} \\
&= \sum_{\vect{s}}{\mu (\vect{s})} \sum_{\vect{u}} \sum_{i=1}^{n} \vect{\pi}(\vect{u} \mid \vect{s}, \vect{\theta})  \cdot \\
& \qquad b\left(\vect{s}^{\neg i}\right) \nabla_{\theta_i} \log{\pi_{i}(a_{i} \mid \tau_{i}, \theta_{i})} \\
&= \sum_{\vect{s}}{\mu (\vect{s})} \sum_{i=1}^{n} \sum_{\vect{u}} \vect{\pi}(\vect{u} \mid \vect{s}, \vect{\theta}) \cdot \\
& \qquad b\left(\vect{s}^{\neg i}\right) \nabla_{\theta_i} \log{\pi_{i}(a_{i} \mid \tau_{i}, \theta_{i})} \\
&= \sum_{\vect{s}}{\mu (\vect{s})} \sum_{i=1}^{n} b\left(\vect{s}^{\neg i}\right) \cdot \\
& \qquad \sum_{\vect{u}} \vect{\pi}(\vect{u} \mid \vect{s}, \vect{\theta}) \nabla_{\theta_i} \log{\pi_{i}(a_{i} \mid \tau_{i}, \theta_{i})} \\
&= \sum_{\vect{s}}{\mu (\vect{s})} \sum_{i=1}^{n} b\left(\vect{s}^{\neg i}\right) \sum_{\vect{u}} \prod_{j=1, j \neq i}^{n}\pi_j(a_j \mid \tau_j, \theta_j) \cdot \\
& \qquad \pi_i(a_i \mid \tau_i, \theta_i) \nabla_{\theta_i} \log{\pi_{i}(a_{i} \mid \tau_{i}, \theta_{i})} \\
&= \sum_{\vect{s}}{\mu (\vect{s})} \sum_{i=1}^{n} b\left(\vect{s}^{\neg i}\right) \sum_{\vect{u}} \prod_{j=1, j \neq i}^{n}\pi_j(a_j \mid \tau_j, \theta_j) \cdot \\
& \qquad \nabla_{\theta_i} \pi_{i}(a_{i} \mid \tau_{i}, \theta_{i}) \\
\end{split}
\end{align*}
\begin{align*}
\begin{split}
&= \sum_{\vect{s}}{\mu (\vect{s})} \sum_{i=1}^{n} b\left(\vect{s}^{\neg i}\right) \sum_{\vect{u}^{\neg i}} \sum_{a_i} \prod_{j=1, j \neq i}^{n}\pi_j(a_j \mid \tau_j, \theta_j) \cdot \\
& \qquad \nabla_{\theta_i} \pi_{i}(a_{i} \mid \tau_{i}, \theta_{i}) \\
&= \sum_{\vect{s}}{\mu (\vect{s})} \sum_{i=1}^{n} b\left(\vect{s}^{\neg i}\right) \sum_{\vect{u}^{\neg i}} \prod_{j=1, j \neq i}^{n}\pi_j(a_j \mid \tau_j, \theta_j) \cdot \\
& \qquad  \sum_{a_i} \nabla_{\theta_i} \pi_{i}(a_{i} \mid \tau_{i}, \theta_{i}) \\
&= \sum_{\vect{s}}{\mu (\vect{s})} \sum_{i=1}^{n} b\left(\vect{s}^{\neg i}\right) \sum_{\vect{u}^{\neg i}} \prod_{j=1, j \neq i}^{n}\pi_j(a_j \mid \tau_j, \theta_j) \cdot \\
& \qquad \nabla_{\vect{\theta}} \sum_{a_i} \pi_{i}(a_{i} \mid \tau_{i}, \theta_{i}) \\
&= \sum_{\vect{s}}{\mu (\vect{s})} \sum_{i=1}^{n} b\left(\vect{s}^{\neg i}\right) \sum_{\vect{u}^{\neg i}} \prod_{j=1, j \neq i}^{n}\pi_j(a_j \mid \tau_j, \theta_j) \cdot \\
& \qquad \nabla_{\vect{\theta}}1 \\
&= \sum_{\vect{s}}{\mu (\vect{s})} \sum_{i=1}^{n} b\left(\vect{s}^{\neg i}\right) \sum_{\vect{u}^{\neg i}} \prod_{j=1, j \neq i}^{n}\pi_j(a_j \mid \tau_j, \theta_j) 0\\
&= 0
\end{split}
\end{align*}

We should expect this result since the baseline term $b\left(\vect{s}^{\neg i}\right)$ is independent of actions---using the same logic provided by Sutton and Barto \cite{sutton2018reinforcement_chp13}---however it requires more manipulation to prove it for the multi-agent case.

%
%

%
%

\section{Non-Health-Based Environments}\label{app:maddpg_benchmark}
	
The environments described in \cref{sec:experiments} were designed to incorporate the concept of system health. It could be argued that these types of environments are outside the scope of what for which MADDPG was orginially designed, and therefore not a fair point of comparison. %
To address this concern, we ran our three variants of multi-agent PPO on the cooperative navigation environment that appears in the original MADDPG paper \cite{lowe2017multi}.

\begin{figure}
	\centering
	\includegraphics[width=0.5\textwidth]{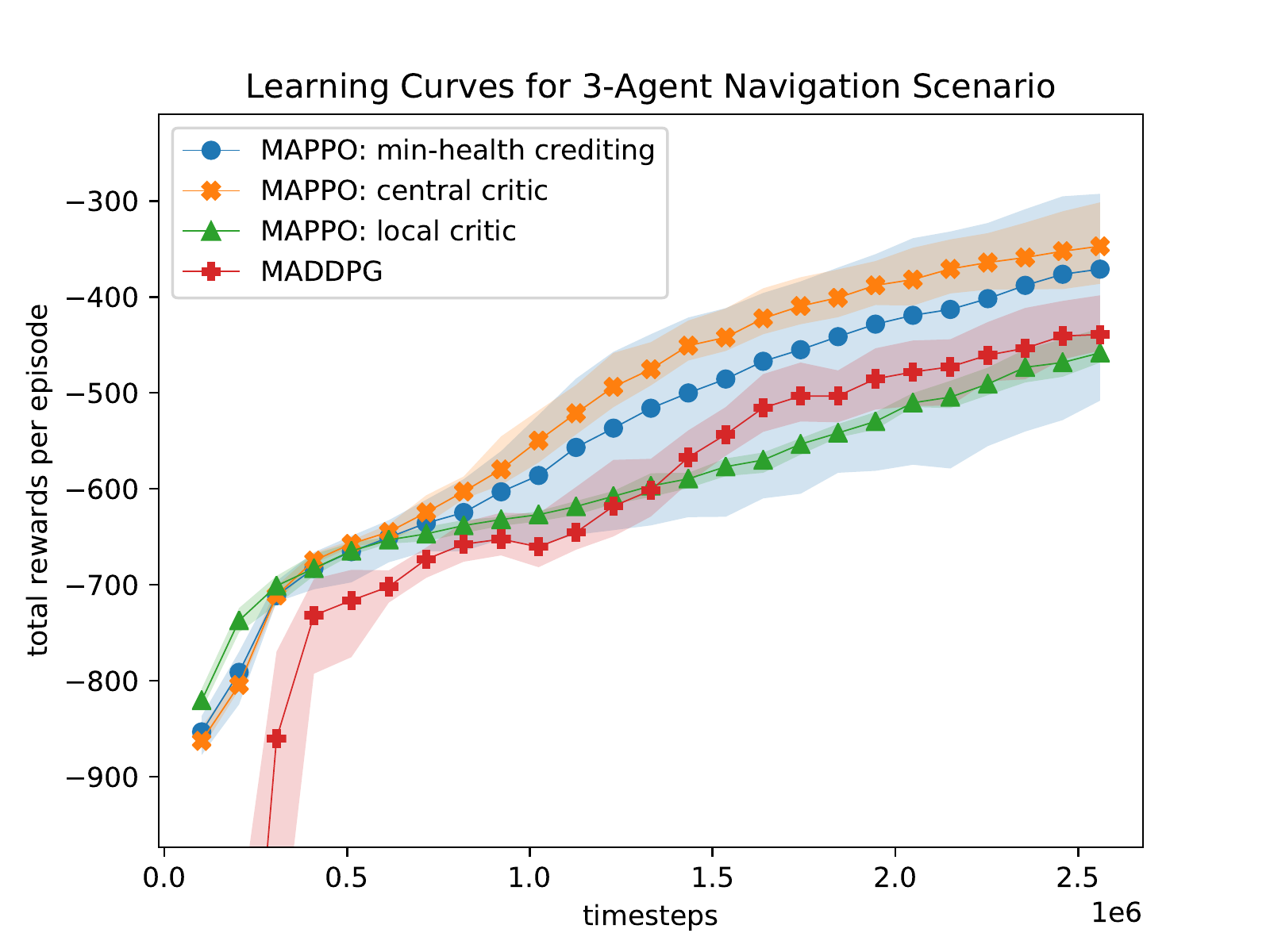}
	\caption{Learning curves for the cooperative navigation problem with 3 agents and \emph{no hazardous landmarks}. Each of the four learning curve is an average over four independent training experiments with the shaded region representing the minimum and maximum bounds of the four training experiments.}
	\label{fig:learning_curves_spread}
\end{figure}

Figure \ref{fig:learning_curves_spread} shows MADDPG successfully learning an effective policy over a 50,000 episode training experiment. 
This learning curve is what we expect to see given that this environment was originally developed in conjunction with MADDPG. It is shown, however, that MADDPG underpeforms all of the MAPPO implementations except the local-critic MAPPO. %
The central-critic, non-crediting variant of MAPPO produces the best performance, outperforming the minimum-health baseline crediting variant. This is contrast to the results for the hazardous communication scenario discussed in Section \ref{sec:experiments}. %
However, this is to be expected due to the fact that the cooperative navigation environment using in \cref{fig:learning_curves_spread} does not encapsulate the concepts of health or risk. Therefore no relevant data is generated that trains the value function network how to estimate the value of the minimum-health baseline in \cref{eqn:min_health_baseline}. %
We see that this causes the minimum-health crediting technique to display high variance between training experiments. The blue shaded region indicates that worst-performing training run on minimum-health MAPPO under-performs all other algorithms, while the best-performing training run out performs all other algorithms.

\end{document}